\documentclass[9pt,technote]{IEEEtran}

\IEEEoverridecommandlockouts                                      

\usepackage{times}
\usepackage{graphicx}
\usepackage{color}
\usepackage{amsmath}
\usepackage{amssymb}
\usepackage[english]{babel}
\usepackage[mathscr]{euscript}
\usepackage{booktabs}

\usepackage[%
hyperfigures=true,%
backref=page,%
pagebackref=true,%
breaklinks=true,%
colorlinks=true,%
citecolor=green,%
linkcolor=blue%
]{hyperref}
\hypersetup{pdfpagemode=UseNone}

\definecolor{grey50}{rgb}{0.5,0.5,0.5}
\newcommand{\FIXME}[1][]{\textcolor{red}{\textbf{FIXME}}\ifthenelse{\equal{#1}{}}{#1}{: \textcolor{grey50}{\emph{#1}}}}
\newcommand{\avik}[1]{{\normalsize{\textbf{({\color{red}Avik:\ }#1)}}}}

\newcommand{\sam}[1]{{\normalsize{\textbf{({\color{blue}Sam:\ }#1)}}}}

\usepackage[inline]{enumitem}
\setlist[enumerate,1]{label=\textit{\alph*)}}
\usepackage{calc}
\usepackage{import}
\usepackage{transparent}
\usepackage{tikz}
\usetikzlibrary{decorations.pathreplacing}

\usepackage{amsthm}
\theoremstyle{plain}
\newtheorem{proposition}{Proposition}
\newtheorem{thm}[proposition]{Theorem}
\newtheorem{corollary}{Corollary}
\newtheorem{lemma}{Lemma}
\newtheorem{definition}{Definition}

\theoremstyle{remark}
\newtheorem{remark}{Remark}

\makeatother

\newcommand{\smm}[1]{\left[\begin{smallmatrix}#1\end{smallmatrix}\right]}
\newcommand{\mat}[1]{\left[\begin{matrix}#1\end{matrix}\right]}
\newcommand{\set}[1]{\left\{ #1 \right\}}
\newcommand{\paren}[1]{\left(#1\right)}
\newcommand{\into}{\rightarrow}

\newcommand{\bbR}{\mathbb{R}}
\newcommand{\bbN}{\mathbb{N}}
\newcommand{\eps}{\varepsilon}
\newcommand{\wt}{\widetilde}
\newcommand{\D}{\mathrm{D}}

\newcommand{\real}{\mathrm{Re}}
\newcommand{\id}{\mathrm{id}}
\newcommand{\restr}[1]{\vert_{#1}}

\newcommand{\sE}{\mathscr{E}}
\newcommand{\sF}{\mathscr{F}}
\newcommand{\sG}{\mathscr{G}}

\newcommand{\sO}{\mathscr{O}}
\newcommand{\sP}{\mathscr{P}}

\newcommand{\sU}{\mathscr{U}}

\newcommand{\sX}{\mathscr{X}}

\newcommand{\rstext}{R}
\newcommand{\rst}{\av{R}}
\newcommand{\vf}{F}
\newcommand{\grd}{\gamma}
\newcommand*\av[1]{%
  \hbox{%
    \vbox{%
      \hrule height 0.5pt 
      \kern0.3ex
      \hbox{%
        \kern-0.2em
        \ensuremath{#1}%
        \kern0.0em
      }%
    }%
  }%
}
\newcommand{\avg}[1]{\overline{#1}}

\newcommand{\ang}{\psi}

\usepackage{verbatim}

\newcommand{\AHS}{averageable system}

\title{A Hybrid Dynamical Extension of Averaging}

\author{Avik De$^\star$, Samuel A. Burden$^\dagger$ and Daniel E.\ Koditschek$^\star$
  \thanks{$^\star$Electrical and Systems Engineering, University of Pennsylvania, Philadelphia, PA, USA.
  {\tt\footnotesize \{avik,kod\}@seas.upenn.edu}.}
  \thanks{$^\dagger$Electrical Engineering, University of Washington, Seattle, WA, USA.
  {\tt\footnotesize sburden@uw.edu}.}
  \thanks{This work was supported in part by 
  the ARL/GDRS RCTA project, Coop.\ Agreement \#W911NF--10-–2−-0016,
  ARO Young Investigator Program Award \#W911NF-16-1-0158 to S. Burden,
  and NSF grant \#1028237.}
}

\makeatother
\begin{document}
\maketitle
\thispagestyle{empty}
\pagestyle{empty}


\begin{abstract}
We extend a  smooth dynamical systems 
averaging technique to a class of hybrid systems with a limit cycle that is particularly relevant to the synthesis of stable legged gaits. After introducing a definition of hybrid averageability sufficient to recover the classical result, we provide a simple illustration of its applicability to legged locomotion and conclude with some rather more speculative remarks concerning the prospects for further generalization of these ideas. 
\end{abstract}


\section{Introduction}
\label{sec:intro}

The emergence of physically motivated and mathematically tractable hybrid  models \cite{johnson_hybrid_2015,burden_event-selected_2014,burden_model_2015}  offers the prospect of  extending classical ideas and techniques of dynamical systems theory for application to new control settings.  In this paper we work at  the intersection of a class of tractable hybrid legged locomotion models \cite{johnson_hybrid_2015} with a class of well--behaved hybrid limit cycle models \cite{burden_model_2015} to generalize an initial result \cite{de_averaged_2015} on the stability of ``averageable'' hybrid oscillators.  Specifically, we extend a classical smooth dynamical averaging technique to a class of hybrid systems with a limit cycle that is particularly relevant to the synthesis of stable gaits.  While our present technical focus precludes more than a simple illustrative example, our aim in pursuing this generalization beyond the narrow sufficient conditions imposed in \cite{de_averaged_2015} is to  enable much broader subsequent applicability of this result to stable composition of dynamical gaits in a new family of legged machines \cite{kenneally_design_2016}. 

In the classical (i.e. smooth) setting, the effect of weak periodic perturbations can be approximated by ``averaging'' continuous 
dynamics over one period \cite[Ch.~4.1]{guckenheimer_nonlinear_1990}. This technique has two benefits for assessing stability:  first, as entailed by the passage to a return map, it reduces the state dimension
by 1 (as the ``averaged'' variable is removed); second, it formally justifies the neglect of certain complicated transient dynamical effects that, on average, do not affect the system's asymptotic behavior. In Def. \ref{def:AverageableHybridSystem}  we generalize such classically 
averageable systems to hybrid models whose continuous flow is punctuated by a single discrete reset. This reset introduces an abrupt change in the system's state that precludes application of the classical result.  However, by incorporating and appropriately analyzing the reset map's contribution to the hybrid system's Poincar\'{e} map, we provide conditions under which the cumulative effect of many iterations of hybrid flow-and-reset dynamics can be approximated using a single classical system's flow. We envision that future work may yield generalizations to systems with multiple domains and time scales, and provide some indications for how such generalizations might be obtained (see Sec.~\ref{sec:conclusion}).

\subsection{Relation to Classical Averaging}
\label{sec:RelationToClassical}

Classical averaging \cite[Ch.~4.1]{guckenheimer_nonlinear_1990}
yields a method of approximating (with error bounds) solutions of the $T$-periodic vector field \eqref{eq:OriginalSystem} using the averaged vector field \eqref{eq:AveragedSystem}.
As in the classical case, our results  guarantee equivalence in stability type to a  simpler approximant (named
the { \em averaged system}) of the system of interest.
Specifically, we show that if the return map of the averaged system has a hyperbolic periodic orbit, then so does the original system, and additionally the linearizations of the return maps are $\eps^2$-close (and thus share the same eigenvalues and eigenvectors to $\sO(\eps)$).

\subsection{Contributions and Organization}

Sec. \ref{sec:hybAvg} introduces an averaging result for hybrid systems in a single domain with non-overlapping guards and fixed time--of--flow (Lemma \ref{lem:CFTavg}).
Subsequently, in Thm. \ref{thm:avg}, we provide a condition under which the
more general case (where the flow time between resetting  is not constant) reduces to the former case (by appropriately redefining the reset map).
Sec. \ref{sec:avgExampleVH} illustrates the utility of this result to the stability analysis of physically interesting models by application of Cor. \ref{cor:computational} to the vertical hopper of Fig. \ref{fig:vertHop}), as well as accompanying simulation to indicate the physical relevance of this theory.
Sec. \ref{sec:conclusion} provides some examples to demonstrate the present limits of this theory, and a conclusion.


\section{Averaging in Single-Mode Hybrid Systems}
\label{sec:hybAvg}

We introduce our class of ``averageable'' hybrid systems  by first imposing the continuous dynamics  properties arising in the classical  setting that entail a flow with a ``fast'' coordinate whose influence over a ``long'' period turns out to be negligible (Def. \ref{def:AverageableHybridSystem}\ref{defSpace}--\ref{defCtsFP}). We then augment the classical setup by introducing a guard and associated  reset map whose base and tangent properties guarantee the overall return map will be hyperbolic (Def. \ref{def:AverageableHybridSystem}\ref{defGuard}--\ref{defRetHyperbolic}) affording the carry-over of stability conclusions to the perturbed  system. The essential result in Lemma \ref{lem:CFTavg} follows from two applications of the implicit function theorem (adapted for this purpose in Lemma \ref{lem:nearbyOrbits}) to track the fixed points of a (first  partially averaged and then fully averaged) approximant to the original Poincar\'{e} map together with calculations (relegated to the Appendix) bearing on the order of magnitude perturbation in the variational flow and tangent maps.  We find it convenient to precede this definition with some calculations in Sec. \ref{sec:avgNonParallel} bearing on the structure of the time-to-reset event map that enter into our sufficient conditions in the ``averageable systems'' Def. \ref{def:AverageableHybridSystem} of Sec. \ref{sec:fftAvg}.


\begin{figure}[t]
\centering
\def\svgwidth{\columnwidth}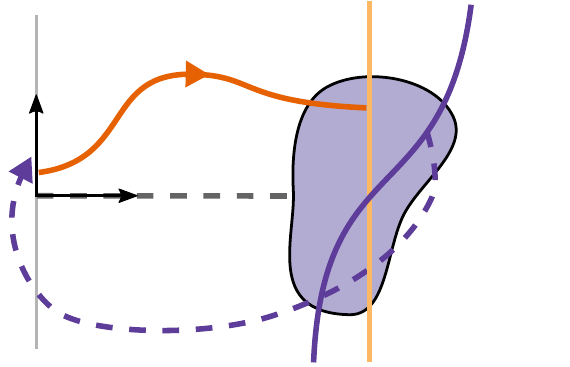
\caption{%
We define a class of \AHS s (Def. \ref{def:AverageableHybridSystem}) with a single domain, fast ($x_1$) and slow ($x_2$) coordinates, with general conditions on the flow (orange), and requirements on fixed points of the flow and the reset, $x^*$. The calculations of \ref{sec:avgNonParallel} show how to construct a new hybrid system with constant flow time (yellow, guard $\av \sG := \{x_1^*\}\times \sX_2$) for any hybrid system with variable flow time (purple, guard $\sG = \gamma^{-1}(0)$) by augmenting with a flow--to--reset map.%
}%
\label{fig:npguard}
\end{figure}

\subsection{Transformation from Variable to Constant Flow Time}
\label{sec:avgNonParallel}

We seek a hybrid dynamical extension of averaging.
This involves introducing a \emph{guard} surface into the state space of an ODE with a cyclic variable that triggers a discrete \emph{reset} of the system state at trajectory crossing events.
The time elapsed between these events generally varies with initial condition, complicating our analysis.
Thus before proceeding, we digress to provide a construction that enables us to transform a system with variable flow time to a system with constant flow time that has equivalent%
\footnote{%
We use this term to denote a correspondence stronger than conjugacy; whereas the flows are indeed conjugate, the return maps are identical.}
Poincar\'e (i.e., flow–-and–-reset)  dynamics.

To that end, let:
$\sX$ be an open subset of Euclidean space;
$\vf$~$:$~$\sX$~$\to$~$T\sX$ be a $C^2$ vector field; 
$\rstext:\sX \to \sX$ be a $C^1$ function;
$\grd: \sX \to \bbR$ be a $C^1$ function such that with $\sG = \gamma^{-1}(0)$ we have%
\footnote{i.e. $F$ is transverse to the codimension--1 submanifold $\sG$}
$(\D\grd \cdot F)|_{\sG} \neq 0$; 
and $x^*\in\sG$ (hence $\grd(x^*) = 0$).
Let $\Phi:\sF\to\sX$ denote the maximal flow associated with $\vf$; recall that $\sF\subset\bbR\times\sX$ is an open set containing $\set{0}\times\sX$~\cite[Ch.~8~\S7]{hirsch_differential_1974}. 
Applying the implicit function theorem~\cite[App.~IV]{hirsch_differential_1974}
\footnote{justified since $\D_1 \grd \circ \Phi = \D\grd \cdot \vf \ne 0$}
to the equation
\begin{align}
\grd \circ \Phi(t, x) = 0
\label{eq:TTI}
\end{align}
with respect to $t$ at $(0,x^*)\in\sF$
yields a $C^1$ time-to-event map $\tau:\sU\to\bbR$ where $\sU\subset\sX$ is a neighborhood of $x^*$~\cite[Ch.~11~\S2]{hirsch_differential_1974}.  
Note that the image of $\tau$ includes negative times.

Now suppose $\sX = \sX_1\times \sX_2$ where 
$\sX_1\subset\bbR$ and
$\sX_2\subset\bbR^n$ are open, and let $x^* = (x_1^*,x_2^*)$. 
Let $\sU_2 = \set{x_2\in\sX_2 : (x_1^*,x_2)\in\sU}$ 
and
define $\rst:\sU_2\to\sX_2$ by
\begin{align}
\forall x_2\in\sU_2 : \rst(x_2) := \pi_2 \circ \rstext \circ \Phi(\tau(x_1^*,x_2), (x_1^*,x_2)).
\label{eq:vftRst}
\end{align}
Note that the derivative (i.e. gradient) of $\tau$ can be computed at $x^*$ by differentiating \eqref{eq:TTI} with respect to $x$%
\footnote{recall that $\D_2\Phi(0, x^*) = I$}
\begin{align}
\D \grd \cdot( \D_1 \Phi \cdot \D \tau + \D_2 \Phi) = 0
\implies 
\D \tau = \frac{-\D \grd}{\D \grd \cdot \vf}.
\label{eq:gradTTI}
\end{align}
Differentiating~\eqref{eq:vftRst} with respect to $x$%
\footnote{recall that $\D_1\Phi(0,x^*) = F(x^*)$}
and substituting using~\eqref{eq:gradTTI} we conclude that
\begin{align}
\D\rst = \D\pi_2 \Big( \D_2\rstext - \frac{\D\rstext \cdot \vf \cdot \D_2 \grd}{\D \grd \cdot \vf} \Big).
\label{eq:npDRgeneral}
\end{align}



\subsection{Averaging for Single-mode Hybrid Systems}
\label{sec:fftAvg}


\begin{definition}[averageable single-mode hybrid system]
\label{def:AverageableHybridSystem}
Given a separation of time-scales parameter $\eps \ge 0$, 
$H = (\sX,\vf,\sG,\rstext,x^*)$
is a
\emph{single-mode hybrid dynamical system averageable at $x^*\in\sX$} if:
\begin{enumerate}[label=(\roman*)]
\item \label{defSpace}
$\sX := \sX_1\times\sX_2$ is a domain comprised of
an open interval around the origin, $\sX_1 \subset \bbR$, containing the ``phase'' coordinate and a topological $n$-dimensional ball $\sX_2 \subset \bbR^n$ containing the remaining coordinate;

\item \label{defCtsDyn}
$F:\sX\to\bbR^{n+1}$ is an $\eps$-parameterized vector field with the form
\begin{align}
\dot x = F(x) = e_1 + \eps \mat{F_1(x) \\ F_2(x)},
\label{eq:OriginalSystem}
\end{align}
where: $e_1$ is the canonical unit vector of $\sX_1 \subset \sX$, and 
$F_1$~$:$~$\sX$~$\to$~$\bbR$, 
$F_2 : \sX \to \bbR^n$ 
are $C^2$ $\eps$-parameterized maps;

\item \label{defCtsFP}
$x^* = (x_1^*, x_2^*)$ is such that 
the averaged vector field $\avg{f}:\sX_2\to\bbR^n$ defined by
\begin{align}
\av{f}(x_2) := \frac{1}{x_1^*} \int_0^{x_1^*} F_2\big(\smm{\sigma\\x_2}\big)\Big\vert_{\eps=0}\,d\sigma
\label{eq:AveragedSystem}
\end{align}
has an equilibrium at $x_2^*$;

\item \label{defGuard}
$\sG\subset\sX$ 
is a codimension-1 embedded submanifold called a \emph{guard} 
that 
(i) contains $x^*$ and
(ii) can be specified (not uniquely) as the zero level-set of a $C^1$ regular function $\grd: \sX \to \bbR$, i.e. $\sG := \grd^{-1}(0)$, and $\D_1\grd|_{\sG} \ne 0$ for any such $\grd$.\footnote{An alternate way to state the last condition is that $\langle w, e_1 \rangle \neq 0$ where $w$ is a vector normal to $\sG$}

\item \label{defReset}
$\rstext : \sG \to \sX$ is an $\eps$-parameterized $C^2$ reset map such that 
\begin{enumerate*}[label=(\alph*)]
\item $\pi_1 \rstext(x) \equiv 0$%
\footnote{When $\pi_1 \circ R$ is constant but nonzero, a shift of the $x_1$ coordinate ensures $R$ satisfies the stated hypothesis. 
When $\pi_1 \circ R$ is not constant,
$R$ can be augmented with a flow-to-event map as in~\ref{sec:avgNonParallel} to satisfy the stated hypothesis.
}%
, and \label{defResetPhase}
\item the resetting rule for the slow coordinates
satisfies $\pi_2 R(x^*) = x_2^*$;
\label{defResetFP} 
\end{enumerate*}
\item \label{defRetHyperbolic}
$\D\av R$, the Jacobian derivative of $\rst$ \eqref{eq:vftRst}, 
has a Taylor series expension in terms of $\eps$ with the form%
\footnote{Note that we impose as a condition that $S_0$ is a constant matrix that does not vary with $x_2$; this property is employed in the proof of Lemma \ref{lem:CFTavg} through Appendix \ref{app:resetLipschitz}.}%
\begin{align}
\D\av R(x_2) = S_0 + \eps S_1(x_2) + O(\eps^2),
\label{eq:DRtaylor}
\end{align}
and at $x_2^*$ we have 
\begin{enumerate*}[label=(\alph*)]
\item $S_0$ is invertible,
\item unity eigenvalues of $S_0$ have diagonal Jordan blocks,
\item \label{defHyperbolic} $S_1 + x_1^* S_0 \D \av f$ is invertible\footnote{Used in the proof of Lemma \ref{lem:CFTavg} through Lemma \ref{lem:nearbyOrbits}.}. 
\end{enumerate*}
\end{enumerate}%
\end{definition}

For ease of exposition, in what follows we will refer to a tuple that satisfies the conditions of Def.~\ref{def:AverageableHybridSystem} simply as an \emph{averageable system}.
We will now specialize to systems with constant flow time to simplify the proof of Lemma~\ref{lem:CFTavg}, then generalize to systems with variable flow time in Thm. \ref{thm:avg}. 

\begin{definition}[{\AHS} with constant flow time]\label{def:fftsmahs}
An {\AHS} $(\sX,\vf,\sG,\rstext,x^*)$
\emph{has constant flow time} if 
\begin{align}
\sG = \av\sG := \{ x_1^* \} \times \sX_2.
\label{eq:constG}
\end{align}
\end{definition}
Note that applying the construction in \eqref{eq:vftRst} to an {\AHS} with constant flow time yields
\begin{align}
\rst(x_2) := \pi_2 \rstext(x_1^*, x_2)
\label{eq:fftRst}
\end{align}
since $\tau \equiv 0$, 
whence the Jacobian $\D\av R = \D\pi_2\D_2 R$ of~\eqref{eq:fftRst} agrees with the general formula \eqref{eq:npDRgeneral} (since $\D_2\gamma = 0$). 

\begin{remark}[Transformation from variable to constant flow time]
\label{rmk:equivCFT}
Let $H = (\sX,\vf,\sG,\rstext, x^*)$ satisfy the conditions of Def.~\ref{def:AverageableHybridSystem}.
Then, following the constructions in~\ref{sec:avgNonParallel}, we define $\av H = (\sX,\vf,\avg{\sG}, (0, \rst), x^*)$ with 
$\avg{\sG}$ as in~\eqref{eq:constG}
and
$\rst$ as in \eqref{eq:fftRst}.
Note that
\begin{enumerate}
\item $\av H$ is an {\AHS} with constant flow time and
\item the flow-and-reset dynamics (i.e. the Poincar\'{e} maps) are equivalent for the two systems:
\begin{align}
\rst\circ\pi_2\circ\Phi(x_1^*,(0,x_2)) = \pi_2\circ \rstext\circ\Phi(\tau(0,x_2),(0,x_2)),
\label{eq:equiv}
\end{align}
where $\tau$ is the time-to-event map for $\sG$, and~\eqref{eq:equiv} holds in a neighborhood of $\pi_2\circ\rstext(x^*)\subset\sX_2$ where both sides of the equation are defined.
\end{enumerate}
We leverage this equivalence to generalize from averaging systems with constant flow time (Lemma \ref{lem:CFTavg}) to averaging systems with variable flow time (Thm. \ref{thm:avg}).
\end{remark}

\begin{remark}[Relation to classical averaging]
Def.~\ref{def:AverageableHybridSystem}\ref{defCtsDyn} specializes  the continous dynamics in the formulation of a general hybrid dynamical system given by \cite{burden_model_2015} to the canonical form for classical averaging \cite{guckenheimer_nonlinear_1990}. From \eqref{eq:OriginalSystem}, we get
\begin{align}
\frac{d x_2}{d x_1} = \frac{\eps F_2(x)}{1 + \eps F_1(x)} =: \eps f(x, \eps),
\label{eq:ddphase}
\end{align}
which~\cite[(4.1.1)]{guckenheimer_nonlinear_1990} regards as a non-autonomous $x_1$-varying dynamical system. 
Note that the averaged vector field~\eqref{eq:AveragedSystem} is the $x_1$-average of \eqref{eq:ddphase} at $\eps = 0$, coinciding with the definition in \cite[(4.1.2)]{guckenheimer_nonlinear_1990}.

If in fact for every $x_2\in\sX_2$ we have $\rst = \id$ then:
\begin{enumerate}
\item we interpret the phase variable as residing in $\left[0, x_1^*\right]/\paren{0 \sim x_1^*}$ (i.e. a circle with circumference $x_1^*$), 
\item Def. \ref{def:AverageableHybridSystem}\ref{defRetHyperbolic} reduces to hyperbolicity of $\av{f}$,
\end{enumerate}
whence we recover the hypotheses necessary for classical averaging \cite[Ch. 4.1]{guckenheimer_nonlinear_1990} and our Lemma~\ref{lem:CFTavg} reduces to \cite[Thm 4.1.1(ii)]{guckenheimer_nonlinear_1990}.
\end{remark}

\begin{remark}[Multiple hybrid modes]
The theoretical statements in this paper all pertain to hybrid systems with a single guard; however, generalizing to cases with multiple disjoint
guards in the same ambient space 
is straightforward: 
if the periodic orbit intersects guards $\sG_1, \ldots, \sG_N$ transversely%
, we require for each $i \in \{1, \ldots, N\}$:
\begin{enumerate}
\item $F_i$, $x^*_i$, $\sG_i$ satisfy Def. \ref{def:AverageableHybridSystem}\ref{defCtsDyn}--\ref{defGuard} in each mode;
\item the composition $R_N \circ \cdots \circ R_1$ satisfies Def. \ref{def:AverageableHybridSystem}\ref{defReset};
\item the product $\prod_i \D \rst_i$ satisfies Def. \ref{def:AverageableHybridSystem}\ref{defRetHyperbolic}.
\end{enumerate}
The case where the hybrid modes reside in different spaces lies outside the scope of this paper (but see Sec. \ref{sec:fwMultDomains} for a discussion).
\end{remark}

\begin{remark}[Relation to smoothing]
Averageable hybrid systems are \emph{smoothable} in the sense that they satisfy the hypotheses of~\cite[Thm.~3]{burden_model_2015}.  Since that result gives a conjugacy to a classical (non-hybrid) vector field and since the smoothing does not affect the $\eps$-dependence of the vector field, it is unsurprising that we are able to extend classical averaging theory to the present hybrid setting. 
The contribution in this paper is the provision of a constructive---in fact, computational (Sections \ref{sec:avgStability}, \ref{sec:avgExampleVH})---method useful for stability analysis of hybrid systems.
\end{remark}


We are now prepared to state and prove our first technical result.\footnote{We develop our results in a manner roughly mirroring \cite[Sec. 4.1]{guckenheimer_nonlinear_1990}. Note that while \cite[Thm. 4.1.1(i)]{guckenheimer_nonlinear_1990} does not in itself assert the existence of a periodic orbit and holds for a flow over an open interval, the next  proposition \cite[Thm. 4.1.1(ii)]{guckenheimer_nonlinear_1990}. Equation \cite[(4.1.2)]{guckenheimer_nonlinear_1990} assumes the existence of a hyperbolic fixed point of the continuous dynamics, which we have extended to the hybrid setting by  positing an averaged continuous time equilibrium state in Def. \ref{def:AverageableHybridSystem}\ref{defCtsFP} together with a suitably fixed reset in Def. \ref{def:AverageableHybridSystem}\ref{defReset}, imposing algebraic properties on its tangent map sufficient for hyperbolicity in Def. \ref{def:AverageableHybridSystem}\ref{defRetHyperbolic}. 
So, our Lemma \ref{lem:CFTavg} provides a similar conclusion as \cite[Thm. 4.1.1(ii)--(iii)]{guckenheimer_nonlinear_1990}, except now in the hybrid system realm.}

\begin{lemma}[Averaging with constant flow time]
Let $(\sX, \vf, \sG, \rstext, x^*)$ be an {\AHS} (Def.~\ref{def:AverageableHybridSystem}) with constant flow time (Def.~\ref{def:fftsmahs}). 
Then for all $\eps > 0$ sufficiently small, 
eigenvalues of the linearization of the flow-and-reset map \eqref{eq:equiv}
are $\sO(\eps^2)$-close to those of 
$\D \av P(x_2^*) := \D\av R(x_2^*) \cdot (I + \eps x_1^* \D \av f(x_2^*))$. 
Thus the asymptotic dynamics of the averaged system approximate those of the original system.
\label{lem:CFTavg}
\end{lemma}
For the proof, we first present a Lemma based on the application of implicit function theorem (IFT) \cite{hirsch_differential_1974} to finding periodic orbits.
%

\begin{lemma}[Nearby fixed points]
Suppose $\sP\subset\bbR^n$ is an open subset, $P_\eps : \sP\to\sP$ 
is an $\eps$--parameterized \footnote{For clarity, in this Lemma we provide an explicit parameterization $P_\eps$ but for the remainder of the text this parameterization is implicit.} $C^2$ map such that $P_0(p^*) = p^*$, and 
the $\eps$-Taylor expansion $\D P_\eps(p^*) = A_0 + \eps A_1 + \sO(\eps^2)$ satisfies
\begin{enumerate}
\item $A_0$ is invertible,
\item unity eigenvalues of $A_0$ have a diagonal Jordan block,
\item $A_1$ is invertible.
\end{enumerate}

\noindent
Then there exists a $C^1$ function $\rho:\sE\into\sP$ defined over an open interval $\sE\subset\bbR$ containing the origin such that $\rho(\eps)$ is a fixed point of $P_\eps$ for all $\eps\in\sE$, 
i.e. $\forall\eps\in\sE : P_\eps(\rho(\eps)) = \rho(\eps)$.
\label{lem:nearbyOrbits}
\end{lemma}

\begin{proof}
If 1 is an eigenvalue of $A_0$, let $m\in\bbN$ be its (algebraic and geometric) multiplicity; otherwise let $m = 0$.
By passing to the Jordan form, without loss of generality we assume $A_0$ has the following block form:
\begin{align}
A_0 = V \mat{I_m & 0 \\ 0 & U} V^{-1}
\label{eq:Kjordan}
\end{align}
where $I_m\in\bbR^{m\times m}$ denotes the $m\times m$ identity matrix and $1$ is not an eigenvalue of $U$.
Now let $E:\bbR\into\bbR^{n\times n}$ denote the block matrix
\begin{align}
\forall\eps\in\bbR : E(\eps) := V \mat{\frac{1}{\eps} I_m & 0 \\ 0 & I_{n-m}} V^{-1}.
\label{eq:Eeps}
\end{align} 
Finally,
define $\zeta_\eps : \sP \to \sP : p \mapsto E(\eps)\paren{P_\eps(p) - p}$%
\footnote{A naive definition of $\zeta_\mathrm{naive} = P_\eps(p) - p$ would fail to have a non-singular Jacobian w.r.t $p$ at $\eps=0$ if $A_0$ has any unity eigenvalues, precluding use of the implicit function theorem. 
In \cite[Thm. 4.1.1 proof]{guckenheimer_nonlinear_1990}
the choice $\zeta_\mathrm{GH} := \frac{1}{\eps}(P_\eps(p) - p)$ sufficed to overcome this obstacle since the absence of a reset map in the classical case implies $A_0 \equiv I$. 
Our $\zeta$ defined here generalizes the same idea to the situation where there is an intruding reset map.},
and observe that:
\begin{enumerate}
\item $\zeta_0 (p^*) = 0$; and
\item $\D_1 \zeta_\eps(p^*)$ is full rank for $\vert \eps \vert$ sufficiently small;
\end{enumerate}
this second point is proven in Appendix \ref{app:Dzeta}.
By the implicit function theorem~\cite[App.~IV]{hirsch_differential_1974},
there exists a $C^1$ 
function $\rho:\sE\into\sP$ defined over a sufficiently small open interval $\sE \subset\bbR$ containing the origin such that $\zeta_\eps( \rho(\eps) ) \equiv 0$, i.e. $\forall\eps\in\sE:P_\eps(\rho(\eps) ) = \rho(\eps)$.
\end{proof}


\begin{proof}[Proof of Lemma~\ref{lem:CFTavg} (based on \cite{guckenheimer_nonlinear_1990})]
The ODE \eqref{eq:OriginalSystem} satisfies all conditions required for the proof of \cite[Thm.~4.1.1(i)]{guckenheimer_nonlinear_1990} on the set $t\in(0,x_1^*)$. 
Construct the change of coordinates $x_2 = y + \eps w(y, t, \eps)$, as in~\cite{guckenheimer_nonlinear_1990}, so that the $\theta$-augmented dynamics of \eqref{eq:OriginalSystem} and \eqref{eq:AveragedSystem} become (respectively) the autonomous systems
\begin{align}
\dot y &= \eps \av{f}(y) + \eps^2 f_1(\theta, y, \eps), &\dot\theta =1,\label{eq:origCC}\\
\dot y &= \eps \av{f}(y), &\dot\theta = 1,\label{eq:avgCC}
\end{align}
where $(\theta, y) \in \sX$ and $f_1$ is a lumped remainder term.
Define $Q : \sX_2\to\sX_2$ and $\av{Q} :\sX_2\to\sX_2$ as the time-$x_1^*$ flows for \eqref{eq:origCC} and \eqref{eq:avgCC} from $x_1 = 0$, i.e. for every $x_2\in\sX_2$ let
\begin{align}
Q(x_2) = \pi_2\circ\Phi(x_1^*,(0,x_2)),\ \av{Q}(x_2) = \pi_2\circ\avg{\Phi}(x_1^*,(0,x_2)),
\label{eq:Qdefs}
\end{align}
where $\Phi$ is the flow for~\eqref{eq:origCC} and
$\avg{\Phi}$ is the flow for~\eqref{eq:avgCC}.
Finally, define
\begin{align}
P := \rst \circ Q,\ \av{P} := \rst\circ \av{Q}.
\label{eq:Pdefs}
\end{align}

Since $x_2^*$ is an equilibrium of~\eqref{eq:avgCC}, $\av{f}(x_2^*) = 0$. Using \cite[pg. 300]{hirsch_differential_1974}, the spatial derivative of the flow around an equilibrium is that of a linear time-invariant system, whence
\begin{align}
\D \av{Q}(x_2^*) = \exp\left(x_1^* \left(\eps \D \av{f}(x_2^*)\right)\right) = I + \eps x_1^* \D\av{f}(x_2^*) + \sO(\eps^2).
\label{eq:flowJacTI}
\end{align}
Note that $x_2^*$ is a fixed point of $\av{Q}$ since
it is both an equilibrium of $\av{P}$ as well as a fixed point of $\rst$ (Def. \ref{def:AverageableHybridSystem}\ref{defResetFP}), 
and therefore
\begin{align*}
\D \av P(x_2^*) = \D \rst(x_2^*) \left(I + \eps x_1^* \D\av{f}(x_2^*) \right) + \sO(\eps^2).
\end{align*}
From \eqref{eq:retMapJac1} and Def. \ref{def:AverageableHybridSystem}\ref{defRetHyperbolic}, we have the Taylor expansion
\begin{align}
\D \av{P}(x_2^*) = S_0 + \eps \left(S_1(x_2^*)+ x_1^* S_0 \D\av{f}(x_2^*)\right) + \sO(\eps^2).
\label{eq:retMapJac1}
\end{align}
First, note that since $S_0$ is invertible (Def.~\ref{def:AverageableHybridSystem}\ref{defRetHyperbolic}), the fixed point is hyperbolic and hence isolated for $\eps > 0$ sufficiently small.
Additionally, we know that
\begin{enumerate}
\item $\av{P}(\rho(0)) = x_2^*$ (from \eqref{eq:Pdefs}, \eqref{eq:flowJacTI} and Def.~\ref{def:AverageableHybridSystem}\ref{defResetFP}),
\item $\D \av{P}(x_2^*)\restr{\eps=0} = S_0(x_2^*)$ is invertible (from Def.~\ref{def:AverageableHybridSystem}\ref{defRetHyperbolic}), and
\item the $\sO(\eps)$ term in \eqref{eq:retMapJac1} is invertible (from Def.~\ref{def:AverageableHybridSystem}\ref{defRetHyperbolic}).
\end{enumerate}
Applying Lemma~\ref{lem:nearbyOrbits}, 
we conclude that
$\av{P}$ has a family of fixed points specified by a map $\rho:\sE\to\sX_2$ satisfying $\rho(0) = x_2^*$.

In Appendix \ref{app:resetLipschitz}, we show that
\begin{align}
\D\rst(\rho(\eps)) = \D\rst(x_2^*) + \sO(\eps^2),
\label{eq:DROeps2}
\end{align}
and in Appendix \ref{app:retMapJac2}, we show that $\D Q(\rho(\eps)) = \D \av Q(x_2^*) + \sO(\eps^2)$. 
Together with \eqref{eq:DROeps2}, we conclude
\begin{align}
\D P(\rho(\eps)) &= (\D\rst(x_2^*) + \sO(\eps^2))(\D \av Q(x_2^*) + \sO(\eps^2)) \nonumber\\
&= \D \av{P}(x_2^*) + \sO(\eps^2).
\label{eq:retMapJacUnAvg}
\end{align}
Since the $\eps$-expansion of \eqref{eq:retMapJacUnAvg} is identical to \eqref{eq:retMapJac1},
we can apply Lemma~\ref{lem:nearbyOrbits} again 
to conclude that $P$ 
has a family of fixed points specified by $\wt\rho:\wt\sE\to\sX_2$ satisfying $\wt\rho(0) = x_2^*$.
Reusing Appendix \ref{app:retMapJac2} and the argument in \eqref{eq:DROeps2} for the first equality below, we see that
\begin{align*}
\D P(\wt\rho(\eps)) &= \D P(\rho(\eps)) + \sO(\eps^2) \stackrel{\eqref{eq:retMapJacUnAvg}}{=} \D \av P(x_2^*) + \sO(\eps^2).
\end{align*}
Thus, we have shown that the eigenvalues of $\D P(\wt\rho(\eps))$ (the linearization of the return map at the fixed point of \eqref{eq:OriginalSystem}) are $\eps^2$-close to eigenvalues of $\D \av{P}(x_2^*)$, which has the simple form \eqref{eq:retMapJac1}.
\end{proof}

\begin{remark}[Lower and upper bounds on $\eps$]
The conclusion of the preceding Lemma is formally valid only for $\eps > 0$ sufficiently small.
However, it may be possible in practice to obtain lower or upper bounds on the allowable range for $\eps$.
\begin{enumerate}
\item Since we have invoked IFT in Lemma \ref{lem:nearbyOrbits}, it is straightforward (if tedious) to bound the size of the neighborhood in which \eqref{eq:OriginalSystem} has a periodic orbit as in \cite[Supplement 2.5A]{abraham_manifolds_1988}.
Alternatively, singular perturbation methods~\cite{tsatsos2006theoretical} may provide lower bounds on values of $\eps$ that ensure the conclusions of Lemma~\ref{lem:nearbyOrbits} hold.
%
\item An obstruction to enlarging the upper bound on $\eps$ appears in our example in Sec. \ref{sec:avgExampleVH}: the quotient \eqref{eq:vertHopVF} is only valid when $\dot x_1 > 0$, which is violated when $\eps > \omega$.
\end{enumerate}
\end{remark}


Combining the constructions in \ref{sec:avgNonParallel} with the conclusions of Lemma \ref{lem:CFTavg}, we obtain conditions under which a general {\AHS} (Def.~\ref{def:AverageableHybridSystem}) may be approximated by its average.
\begin{thm}[Averaging with variable flow time]
Let $ (\sX, \vf, \sG, \rstext, x^*)$ be an {\AHS} (Def.~\ref{def:AverageableHybridSystem}).
Then for all $\eps > 0$ sufficiently small, 
eigenvalues of the linearization of the flow-and-reset map \eqref{eq:equiv}
are $\sO(\eps^2)$-close to those of 
$\D \av P(x_2^*) := \D\av R(x_2^*) \cdot (I + \eps x_1^* \D \av f(x_2^*))$. 
Thus the asymptotic dynamics of the averaged system approximate those of the original system.
\label{thm:avg}
\end{thm}

\begin{proof}
Let $\av H=(\sX,\vf, \{x_1^*\} \times \sX_2, (0, \rst))$ be as in Remark \ref{rmk:equivCFT}. 
Note that $\av H$ is an {\AHS} (Def.~\ref{def:AverageableHybridSystem}) with constant flow time (Def.~\ref{def:fftsmahs}), whence we can apply Lemma \ref{lem:CFTavg}. 
We conclude that  $\D \av R \circ \D Q (\wt x_2^*) = \D\av R \circ \D \av Q (x_2^*) +\sO(\eps^2)$, where $\wt x_2^*$ is the fixed point of $\av R \circ Q$. As in \ref{sec:avgNonParallel}, the linearization $\D \av R$ is identical for $H$ and $\av H$; we conclude that the linearization of the return map of $H$ satisfies $\D P (\wt x_2^*) = \D\av P (\wt x_2^*) + \sO(\eps^2) = \D\av R \circ \D\av Q (x_2^*) + \sO(\eps^2)$.
\end{proof}

\subsection{Stability of Orthogonally Reset Averageable Systems}
\label{sec:avgStability}



We focus in this section on the 
case where $S_0$ is an orthogonal matrix. 
Our motivation for this comes from the form of the $\eps$-parameterized return map \eqref{eq:retMapJac1}. If $S_0$ has eigenvalues which are not strictly on the unit circle, the asymptotic behavior is dominated by $S_0$ for small $\eps$ (rendering the continuous dynamics irrelevant). In this subsection we examine contractiveness of $\D P^T \D P$ as a sufficient condition for stability, in which case the relevant property of $S_0$ is orthogonality, a special case of having eigenvalues on the unit circle. Even though we only discuss a scalar example in this paper (Sec. \ref{sec:avgExampleVH}), orthogonally reset systems appear widely in hybrid system models for locomotion (e.g. \cite{altendorfer_stability_2004,seipel_running_2005,de_parallel_2015}).

Using Thm. \ref{thm:avg}, we expose in a simple formula (Cor.~\ref{cor:computational}) the relative contributions of the continuous (specified by $\vf$) and discrete (specified by $\rstext$) dynamics to the overall (in)stability of the {\AHS}.

\begin{definition}[{\AHS} with orthogonal reset]
An {\AHS}
$(\sX,\vf,\sG,\rstext,x^*)$  
\emph{has orthogonal reset}
if:
\begin{enumerate}
\item $S_0$ in \eqref{eq:DRtaylor} is an orthogonal matrix;
\item $W:= S_0 \cdot S_1 + x_1^* \D \av f$ is invertible at $x_2^*$.
\end{enumerate}%
\label{def:urhs}
\end{definition}

\begin{remark}
The first condition in the preceding definition implies that $S_0$ is invertible.
The second condition ensures Def. \ref{def:AverageableHybridSystem}\ref{defRetHyperbolic}\ref{defHyperbolic} is satisfied.
Thus orthogonal resetting provides one route to ensure the hypotheses of \ref{def:AverageableHybridSystem}\ref{defRetHyperbolic} are satisfied.
\end{remark}



\begin{corollary}[Stability of {\AHS} with orthogonal reset]
  An {\AHS} $(\sX,\vf,\sG,\rstext, x^*)$ with orthogonal reset has an exponentially stable periodic orbit if $W(x_2^*)^T + W(x_2^*) \prec 0$ (negative definite), where $W$ is as defined in Def. \ref{def:urhs}.%
\label{cor:computational}
\end{corollary}

\begin{proof}
The linearization of the averaged return map at $x_2^*$ is $\D \av{P} = \D \av{R} \cdot \D \av{Q} = (S_0 + \eps S_1)(I + \eps x_1^* \D \av{f}) = S_0 + \eps (S_1 + x_1^* S_0 \cdot \D \av{f}) + \sO(\eps^2) =  S_0(I + \eps W) + \sO(\eps^2)$.

For arbitrary unit vector $v$,
\begin{align*}
&\Vert S_0 (I+ \eps W)v\Vert^2 - \Vert v \Vert^2 \\
&= v^T (I + \eps(W^T + W)) v - v^T v + \sO(\eps^2) \\
&= \eps v^T (W^T+W) v  + \sO(\eps^2).
\end{align*}
For small $\eps > 0$, we see that the right hand is negative
since $W^T + W \prec 0$ by assumption. 
Thus $\D \av{P}$ is a contraction, whence the return map $\av{P}$ is locally exponentially stable.
Using Thm. \ref{thm:avg}, we conclude for $\eps > 0$ sufficiently small that the return map of the original system $P = \rstext \circ Q$ has an exponentially stable fixed point nearby.
\end{proof}


\section{Application: 1DOF Vertical Hopper}
\label{sec:avgExampleVH}

\begin{figure}[t]
\centering
\def\svgwidth{0.5\columnwidth}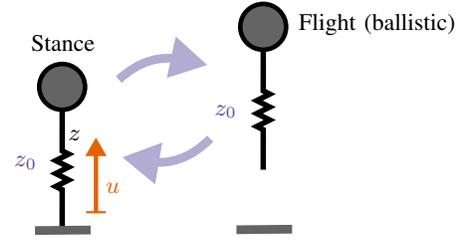
\caption{A vertical hopper model for the Sec. \ref{sec:avgExampleVH} example.}
\label{fig:vertHop}
\end{figure}

The application domain that motivated the preceding theoretical developments is legged locomotion on land.
A well-known model for running is that of a 
mass suspended on a massless leg by a (physical or virtual) spring \cite{geyer_compliant_2006,Blickhan_Full_1993}---the so-called Spring-Loaded Inverted Pendulum (SLIP)~\cite{Saranli_Schwind_Koditschek_1998}.
Variants of this model have been analyzed extensively in the literature, e.g. see \cite{koditschek_analysis_1991} for analysis of a one degree-of-freedom (1DOF) restriction of this planar point-mass model whose energizing input is inspired by the empirically successful strategies reported in \cite{raibert_legged_1986}. 
Informed by the structure of the averaging results of Sec.~\ref{sec:hybAvg}, we propose an alternative energization scheme 
(intuitively similar but physically distinct from~\cite{raibert_legged_1986}) 
and apply Thm.~\ref{thm:avg} to establish an analogous stability result. 

The physical model is illustrated in Fig.~\ref{fig:vertHop}: a unit mass restricted to travel along its vertical axis with an attached massless leg. The ``nominal'' leg length relative to a spring is $z_0$, and the actual leg length is $z$. 
In flight the system goes through a ballistic motion, whose effect on the state can---without any loss of generality or accuracy---be incorporated into the reset map.

\subsection{Averageable Vertical Hopper Model}
In the absence of damping and actuation, the spring-mass hopper exhibits constant stance duration~\cite{raibert_legged_1986}.
Introducing even small amounts of damping or actuation generally leads to variable stance durations, and this effect can influence the system's stability properties. 
In what follows, we show that introducing viscous drag and periodic forcing results in stance duration that varies \emph{weakly} in a sense made precise in~\eqref{eq:vhWeaklyVarying}, account for this effect in the linearized return map~\eqref{eq:vhResetFFT}, and assess stability of hopping in the resulting nonconservative system.

\subsubsection{Continuous dynamics}

Introducing viscous drag $-\eps \beta \dot{z}$ and actuation $u$,
the hopper's equations of motion are
\begin{align}
 \ddot z = \begin{cases}
u -  g - \eps  \beta \dot z & \text{(stance);}\\
- g & \text{(flight).}
\end{cases}
\end{align}

We consider the empirically-motivated weakly nonlinear periodic energization strategy from~\cite[eq. (8)]{de_averaged_2015} in stance, 
which involves defining phase-energy coordinates $x := (\ang,a)$ (where $x_1 = \ang$ denotes the ``phase'' coordinate of Def. \ref{def:AverageableHybridSystem}, and $x_2 = a$ denotes the remaining coordinate) such that
\begin{align}
a \sin\ang = z_0 - z,\qquad a \omega \cos\ang = -\dot z.
\label{eq:vhEPcoords}
\end{align}
 With this notation in force, and having met the requirements of Def. \ref{def:AverageableHybridSystem}\ref{defSpace} for $\ang$ in some open interval%
\footnote{%
It will become clear in \eqref{eq:vertHopAvgFP} that $\Psi \in \bbR$ lies within a single cycle, $0< \Psi < 2 \pi$.}
$\sX_1 := (-\Psi, \Psi) \subset S^1$ and $a \in \sX_2 := \overline{\bbR_+}$,   the feedback law from~\cite[eq. (8)]{de_averaged_2015} becomes
\begin{align}
u(x) :=  \left(g + \omega^2 a \sin\ang - \eps \omega k \cos\ang \right),
\label{eq:vhControlStance}
\end{align}
where the first two summands can be thought of as instantiating a virtual Hooke's law spring, and the last summand's negative damping\footnote{From \eqref{eq:vhEPcoords}, the last term of \eqref{eq:vhControlStance} is $-\eps k \omega \cos\ang = (\eps k/a) \dot z$ (forcing in the direction of velocity, but normalized by $a$).} term serves to supply the system with energy \cite{de_parallel_2015,secer_control_2013}.
In the phase-energy coordinates, the closed-loop dynamics are given by \cite[Eqn.~(10.36)]{khalil_nonlinear_2002}
\begin{align}
\dot x = \mat{\omega \\ 0} + \eps \mat{\frac{1}{\omega a} v(x) \sin\ang \\ -v(x) \cos\ang},
\label{eq:vertHopVF}
\end{align}
where $v(\ang, a):= \omega (a\beta - k)\cos\ang$.
With
\begin{align}
x_1^* = \pi,\qquad x_2^* = k/\beta,
\label{eq:vertHopAvgFP}
\end{align}
a straightforward computation yields
\small
\begin{align}
\int_0^\pi F_2\Big(\smm{\ang\\x_2^*}\Big) d \ang = -\int_0^\pi v(x) \cos\ang d\ang = \frac{\pi(k-x_2^* \beta)}{2} = 0,
\end{align}
\normalsize
so $x^*$ is a fixed point of the averaged vector field
\begin{align}
\av{f}(a) = \frac{k-a\beta}{2\omega},
\label{eq:vertHopAvgVF}
\end{align}
satisfying Def. \ref{def:AverageableHybridSystem}\ref{defCtsDyn}--\ref{defCtsFP}.

\subsubsection{Guard set}
\label{sec:vhGuard}

The guard set is defined by the physical liftoff event when the normal force at the toe-ground interface goes to 0, i.e. $u -  g - \eps  \beta\dot r = 0$. From \eqref{eq:vhControlStance}, this occurs at the zeros of
\begin{align}
\grd(x) := \omega \tan\ang - \eps \left(\frac{k}{a} - \beta\right).
\label{eq:vertHopGuard}
\end{align}
Note that $\grd(x^*) = 0$ at the fixed point of the averaged vector field \eqref{eq:vertHopAvgVF}, satisfying Def.~\ref{def:AverageableHybridSystem}\ref{defGuard}.

\subsubsection{Reset map}
\label{sec:vhReset}

As in \cite{de_averaged_2015}, the massless in-flight leg is reset to its nominal length, $z_0$. It follows from \eqref{eq:vhEPcoords} that 
the touchdown phase, $\ang$ is identically zero since $z = \rho$ at the touchdown event. 
Noting from \eqref{eq:vhEPcoords} that $\dot z = a \omega$ at touchdown, and recalling that the mechanical energy $\dot z^2/2 + g z$ is conserved in flight,
we can solve for $a$ at touchdown, yielding
\begin{align}
R(x) = \mat{0 \\ \sqrt{a^2\cos^2\ang - 2 g a\sin\ang/\omega^2} }.
\label{eq:vertHopReset}
\end{align}
Note that $\pi_1 R \equiv 0$ and $\pi_2 R(x^*) = x_2^*$, satisfying Def. \ref{def:AverageableHybridSystem}\ref{defReset}.

\begin{figure}[t]
\centering
\def\svgwidth{\columnwidth}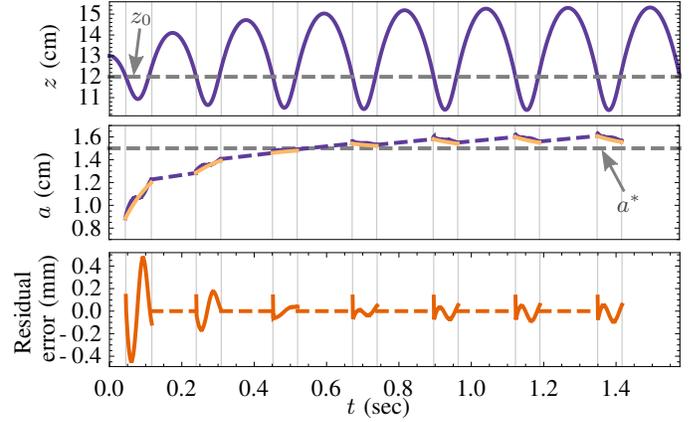
\caption{\textbf{(top)} displacement of vertical hopper in physical $z$ coordinates (thin vertical lines separate stance and flight); 
\textbf{(middle)} abstract energy coordinate $a$ \eqref{eq:vhEPcoords} in purple (dashed: flight), and, in gold, the equivalent continuous dynamical system \eqref{eq:vertHopAvgVF} over several hops.
\textbf{(bottom)} residual error in the $a$ coordinate between trajectories of the averaged and original systems.
}
\label{fig:vhavgplot}
\end{figure}


\subsection{Averaging the Vertical Hopper Model}

With $\sX_1$ and $\sX_2$ as defined above, $\sX = \sX_1\times\sX_2$,
$\vf$ as in~\eqref{eq:vertHopVF},
$\sG = \gamma^{-1}(0)$ with $\gamma$ as in~\eqref{eq:vertHopGuard},
$\rstext$ as in~\eqref{eq:vertHopReset},
and $x^*\in\sX$ as in~\eqref{eq:vertHopAvgFP},
the system $H = (\sX,\vf,\sG,\rstext,x^*)$ satisfies conditions \ref{defSpace}--\ref{defReset} from Def.~\ref{def:AverageableHybridSystem}. 
From \eqref{eq:vertHopGuard}, 
\begin{align}
\D\grd(x^*) = \mat{\sec^2\pi & \frac{\eps\beta^2}{k\omega}}  = \mat{1 & \frac{\eps\beta^2}{k\omega}},
\label{eq:vhWeaklyVarying}
\end{align}
whence by \eqref{eq:vertHopVF} we see that $\D\grd \cdot F(x^*) = \omega + \sO(\eps^2)$.  
Noting further that $\D R \cdot F(x^*) = g/\omega$, we find that \eqref{eq:npDRgeneral} simplifies to
\begin{align}
\D\rst &= 1 -\frac{\eps g \beta^2}{k \omega^3},
\label{eq:vhResetFFT}
\end{align}
where the second summand was introduced by the variable flow time, and would have been neglected if a constant stance duration approximation were employed. 
Using  \eqref{eq:vertHopAvgVF}  we can check that 
\begin{align}
W := S_1 + \pi\D\av f = -\frac{g \beta^2}{k\omega^3} -\frac{\beta \pi}{2 \omega} < 0,
\end{align}
satisfying Def. \ref{def:AverageableHybridSystem}\ref{defRetHyperbolic}. 
We have now checked all the conditions of Thm. \ref{thm:avg}, and from \eqref{eq:vhResetFFT} it is clear that 
\begin{enumerate*}
\item $S_0 = 1$ satisfies the ``orthogonality'' condition, and
\item $W < 0$ satisfies the rank condition
\end{enumerate*}
of Def. \ref{def:urhs}
We conclude from Cor. \ref{cor:computational} that the vertical hopper has a stable fixed point that is $\eps$-close%
\footnote{Practitioners may wish to note that $\eps$-closeness in state corresponds to $\eps$-closeness in energy in mechanical systems like this hopper.}
to \eqref{eq:vertHopAvgFP}.

\begin{remark}[approximating continuous control with discrete steps]
Note that the averaged vector field \eqref{eq:vertHopAvgVF} has the form of a proportional controller on total energy.  Thus Thm \ref{thm:avg} enables us to conclude that the cumulative control effect on body height from multiple isolated steps through a second-order ODE is approximately equivalent to that of a first-order ODE acting on body height (as shown in Fig.~\ref{fig:vhavgplot}(middle)).
\end{remark}


\subsection{Simulation Results}
\label{sec:vhSim}
Using parameters $\omega = 50$ rad/s, $k = 0.4$ N-s/m$^2$, $\beta = 10$ N/(m/s) and $\eps = 2$, numerical simulations%
\footnote{with \href{https://www.wolfram.com/mathematica/}{Mathematica 10}, using \href{http://reference.wolfram.com/language/ref/NDSolve.html}{NDSolve} and \href{https://reference.wolfram.com/language/ref/WhenEvent.html}{WhenEvent}}%
of the vertical hopper show that
\begin{enumerate}
\item the fixed point of the averaged system is approximately 0.15 mm away from the original system's fixed point (Fig. \ref{fig:vhavgplot} middle, difference between purple $a$ and dashed gray $a^*$), and
\item the residual error between trajectories of the averaged and original systems are an order of magnitude smaller than $a^* = k/\beta = 0.04 $ m (Fig.~\ref{fig:vhavgplot}, bottom).
\end{enumerate}


\section{Conclusions and Future Work}
\label{sec:conclusion}

This paper presents, to the best of our knowledge, the first instance
of a generalization to hybrid systems of a classical averaging result. Thus, Thm. \ref{thm:avg} joins a growing body of cases wherein suitably  constructed hybrid systems \cite{burden_event-selected_2014,burden_model_2015,ElderingJacobs2016,AmesSastry2006,WesterveltGrizzle2003,posa_stability_2016}
admit an  appropriately restated version of classical dynamical systems
results,  with useful applications to new engineering settings.
Application to a simple 1DOF model relevant to legged locomotion (Sec.~\ref{sec:avgExampleVH}) indicates that stability analyses of limit cycles in higher dimensional systems \cite{kenneally_design_2016,de_parallel_2015} could be greatly simplified, in analogy to the simple construction afforded by \cite{de_averaged_2015} relative to the initial controllers of \cite{de_parallel_2015}. This is an avenue of ongoing research being undertaken by the authors, and would add to the large body of emerging engineering--motivated  research to develop approximations of the behavior of nonlinear dynamical systems near reference trajectories
\cite{wu_variation-based_2015} (limit cycles in the case of this paper).


We conclude with some examples that demonstrate limitations of the present theory, and in doing so, motivate future theoretical work. 

\subsubsection{Extension to multiple domains}
\label{sec:fwMultDomains}

Intuitively, conditions \ref{def:AverageableHybridSystem}\ref{defCtsDyn}--\ref{defReset} together ensure that all non--phase coordinates vary slowly with respect to the phase. Robotic systems in steady--state operation---with asymptotically stable limit cycles---are one (important) class of systems our results apply to, but by no means the only.

Generalizing the results herein to hybrid systems with multiple modes or overlapping guards presents a number of challenges. With multiple domains, there is no privileged set of coordinates shared across disjoint portions of state space, so it is not obvious how to generalize conditions \ref{defCtsDyn} and \ref{defGuard} in Def. \ref{def:AverageableHybridSystem}. 
With overlapping guards, the return map is generally discontinuous~\cite[Table~3]{RemyBuffinton2010} or at least nonsmooth~\cite[Sec.~4.2]{burden_event-selected_2014}, so it is not obvious how to generalize conditions \ref{defReset} and  \ref{defRetHyperbolic} in Def. \ref{def:AverageableHybridSystem}.

\subsubsection{Effects of $\eps$-perturbation}

As discussed in the remarks after Def.~\ref{def:AverageableHybridSystem}, large--$\eps$ limits of classical averaging conclusions can be found in the literature (e.g. \cite[pg. 17]{tsatsos2006theoretical}). 
Numerically, as well as from our empirical experience in \cite{de_averaged_2015}, the vertical hopper example in Sec.~\ref{sec:avgExampleVH} retains the asymptotic behavior of \eqref{eq:vertHopAvgVF} for large $\eps$ (Sec. \ref{sec:vhSim}). 

\subsubsection{Rank condition in Def. \ref{def:AverageableHybridSystem}\ref{defRetHyperbolic}\ref{defHyperbolic}}

We emphasize that this condition is necessary for averaging along the lines of hyperbolicity in the classical theory \cite[Thm. 4.1.1]{guckenheimer_nonlinear_1990}.  Indeed, consider $H = (\sX, \vf, \sG, \rstext, x^*)$ with 
\[\sX = \bbR^2,\ F = \mat{1\\-\eps x_2},\ \sG = \{x_1^*\} \times \bbR,\ R = \mat{0\\ x_2 + \eps x_1^* x_2}\]
and $x^* = (x_1^*, 0)$ for any $x_1^* > 0$, and observe that $\av{f} = -x_2$ and $\D \av{f} = -1$. 
Since $S_1 + x_1^* \D \av{f} = 0$, $H$ violates Def.~\ref{def:AverageableHybridSystem}\ref{defRetHyperbolic}\ref{defHyperbolic}.  Note that the linearization of the return map,
\begin{align}
\D\rst \cdot \D\av{Q}(x) = (1 + \eps x_1^*) (1 - \eps x_1^*) + \sO(\eps^2) = 1 + \sO(\eps^2),
\end{align}
is not hyperbolic to $\sO(\eps)$, so we cannot assess stability using Thm.~\ref{thm:avg}. 


\subsubsection{Multiple fast coordinates}
\label{sec:concMultFast}

The form of \eqref{eq:OriginalSystem} does not apply directly to systems where there are multiple fast coordinates. However, we plan (in future work) to exploit slow phase difference dynamics, as previously demonstrated in  the classical averaging context \cite{proctor_phase-reduced_2010}.


%

\appendix


\subsubsection{Calculation of $\D_1\zeta$ for proof of Lemma \ref{lem:nearbyOrbits}}
\label{app:Dzeta}

Using \eqref{eq:Kjordan} and \eqref{eq:Eeps}, we can calculate
\begin{align*}
&V^{-1} \D_1 \zeta(p^*,\eps) V \\
&= \mat{\frac{1}{\eps} I_m & \\& I_{n-m}} \left(\mat{0 & \\ & U - I_{n-m}} + \eps V A_1 V^{-1}\right).
\end{align*}
In the limit $\eps \to 0$, the top $m$ rows have rank $m$ since $A_1$ is full rank, 
while the bottom $n-m$ rows evaluate to $U - I_{n-m}$; since $U$ has no unity eigenvalues, $U - I_{n-m}$ is also full rank.

For $\eps \neq 0$, the argument is unchanged for the first $m$ rows. For the bottom rows, the entries of $U - I_m$ dominate those of $\eps V A_1 V^{-1}$, and so by continuity of eigenvalues with matrix entries, the right hand side is full rank for sufficiently small $\vert \eps \vert$.



\subsubsection{Calculation of $\D\rst(\rho(\eps))$ for \eqref{eq:DROeps2}}
\label{app:resetLipschitz}

Using the $\eps$--expansion of $\D\rst$ in Def. \ref{def:AverageableHybridSystem}(\ref{defRetHyperbolic}),
\begin{align*}
\Vert \D\rst(\rho(\eps)) - \D\rst(\rho(0)) \Vert &\le 
\eps L_1 \Vert \rho(\eps) - \rho(0) \Vert = \sO(\eps^2), 
\end{align*}
where $L_1$ is the Lipschitz constant for $S_1$, and we used the fact that $S_0(\rho(\eps)) = S_0(\rho(0))$ as assumed in Def. \ref{def:AverageableHybridSystem}\ref{defRetHyperbolic}.

\subsubsection{Calculation of $\D Q(\rho(\eps))$ for \eqref{eq:retMapJacUnAvg}}
\label{app:retMapJac2}

The time-varying system~\eqref{eq:origCC} can reinterpreted as time-invariant in $x := (\theta, b)$ coordinates,
\begin{align}
\dot x = \mat{ 1 \\ \eps \av f(x_2) + \eps^2 f_1(x, \eps)} =: \wt \vf(x).
\end{align}
Let the time-$t$ flow of the $\wt\vf$ system be denoted by $\wt\Phi_t(x)$. We follow \cite[pg. 300]{hirsch_differential_1974} to compute the spatial Jacobian the flow. Note that $\D Q(\rho(\eps)) = \pi_2 \D_2 \wt\Phi_t(0, \rho(\eps))$ by definition, where $\pi_2$ is the projection on to the second component. Let $\wt x := (0, \rho(\eps))$. As in \cite{hirsch_differential_1974},
\begin{align}
\frac{d}{d t} \D \wt\Phi_t(\wt x) = A(t) \D \wt\Phi_t(\wt x),
\label{eq:LTVForJac}
\end{align}
where
\begin{align}
A(t) := \D \wt \vf(\wt\Phi_t(\wt x)) = \mat{0&0\\\eps^2 \D_1 f_1 & \eps \D \av f + \eps^2 \D_2 f_1}.
\label{eq:AofLTV}
\end{align}
The solution of this time-varying linear system from initial condition $\D \wt \vf(\wt\Phi_0(\wt x)) = I$
can be computed using the Peano-Baker series.
Since we are only interested in the lower right block of \eqref{eq:AofLTV},
\small
\begin{align}
&\D Q(\rho(\eps))
= I + \int_0^{x_1^*} \D \av f(\wt\Phi_t(\wt x)) d t + \sO(\eps^2) \nonumber\\
=&\,  I + \eps \int_0^{x_1^*} \left[ (\D \av{f} (\wt\Phi_t(\wt x)) - \D \av{f} (x_2^*)) + \D \av{f} (x_2^*) \right]d t + \sO(\eps^2) \nonumber\\
=&\,  I + \eps x_1^*\D\av{f}(x_2^*) + \eps \int_0^{x_1^*} \left[ \D \av{f} (\wt\Phi_t(\wt x)) - \D \av{f} (x_2^*) \right] d t + \sO(\eps^2),
\label{eq:peanoBaker}
\end{align}
\normalsize
where it is understood that $\D\av f$ only takes the second component of $\wt\Phi_t(\wt x)$ as argument (notation overloaded for brevity).

By continuity of the flow with respect to initial conditions~\cite[Thm~\S4]{hirsch_differential_1974}, we know that the $\eps$-perturbation of the initial condition, $(0, \wt x_2^*) = (0, \rho(\eps))$ from the usage of Lemma \ref{lem:nearbyOrbits} to find a fixed point of $\av{P}$, only affects the nominal solution $x_\mathrm{nom}(t) \equiv (0, x_2^*)$ in an $\sO(\eps)$ way, $\pi_2\wt\Phi_t(\wt x) = x_2^* + \sO(\eps)$.
Additionally, since   $\D\av{f}$ is Lipschitz continuous,
\[
D\av{f} \circ \pi_2 \wt{\Phi}(0,x_2^* + O(\epsilon))
= D\av{f}(x_2^*) + L\cdot O(\epsilon),
\]
where L is the Lipschitz constant, and \eqref{eq:peanoBaker} simplifies to 
\begin{align}
\D Q(\rho(\eps)) = I + \eps x_1^* \D\av{f}(x_2^*) + \sO(\eps^2).
\label{eq:flowJacToOeps}
\end{align}

\bibliographystyle{IEEEtrannourl}
\bibliography{refs}

\end{document}